\documentclass{article}
\usepackage[utf8]{inputenc}
\usepackage{authblk}
\usepackage[margin=1in]{geometry}
\usepackage[square,sort,comma,numbers]{natbib}
\usepackage{amsfonts}
\usepackage{amsmath}
\usepackage{amsthm}
\usepackage{amssymb}
\usepackage{mathtools}
\usepackage{booktabs}

\newtheorem{prop}{Proposition}

\title{\textbf{Pick Your Battles:  Interaction Graphs\\ 
as Population-Level Objectives for Strategic Diversity}}

\author[1]{Marta Garnelo\thanks{Corresponding Author: garnelo@deepmind.com}}
\author[1]{Wojciech Marian Czarnecki}
\author[1]{Siqi Liu}
\author[1]{Dhruva Tirumala}
\author[1]{Junhyuk Oh}
\author[1]{Gauthier Gidel}
\author[1]{Hado van Hasselt}
\author[1]{David Balduzzi}
\affil[1]{DeepMind}

\date{}

\begin{document}

\maketitle

\begin{abstract}
Strategic diversity is often essential in games: in multi-player games, for example, evaluating a player against a diverse set of strategies will yield a more accurate estimate of its performance. 
Furthermore, in games with non-transitivities diversity allows a player to cover several winning strategies.
However, despite the significance of strategic diversity, training agents that exhibit diverse behaviour remains a challenge.
In this paper we study how to construct diverse populations of agents by carefully structuring how individuals within a population interact. 
Our approach is based on \emph{interaction graphs}, which control the flow of information between agents during training and can encourage agents to specialise on different strategies, leading to improved overall performance. 
We provide evidence for the importance of diversity in multi-agent training and analyse the effect of applying different interaction graphs on the training trajectories, diversity and performance of populations in a range of games. This is an extended version of the long abstract published at AAMAS~\citep{garnelo2021pick}.
\end{abstract}

\section{Introduction}
Most interesting real-world games and tasks involve separate and potentially competing objectives (are multi-agent in nature) and do not admit a single winning strategy that beats all (i.e. have some non-transitive element).
Training agents that master these types of games poses several challenges. For example, in multi-agent games performance is only defined in relation to other players rather than absolutely. Additionally, as a result of the non-transitive nature of these games, performance against one opponent can often be uninformative or even misleading about performance against other opponents. 

In this paper we show that a number of the challenges associated with non-transitive multi-agent environments can be addressed by ensuring the strategy space of our agents is maximally diverse. 
At a higher level, agents that are strategically diverse are able to estimate other agents' performance more accurately, are better opponents to train other agents against and, overall, perform better at test time.

Given the importance of strategic diversity we thus pose the question of how to methodically train such diverse agents. 
A key observation is that training the same agent against different opponents results in varying exploration behaviour during training as well as different final performance of said agent. 
We use this insight to our advantage and train strategically diverse agents by systematically selecting the opponents that an agent encounters during training.

We consider two-player games and train populations of agents by setting each agent against a specific subset of the rest of the population.
In order to characterise these different sets of opponents we introduce the notion of \emph{interaction graphs}, which describe the training interactions between agents in a population.
Depending on the properties of the graph the resulting population will exhibit varying levels of strategic diversity and performance.
We study this effect for a number of different graphs on a modified version of Rock-Paper-Scissors, Blotto and Starcraft.

 It is important to note, that the setup analysed here is qualitatively different from a standard approach, where a best response is found with respect a distribution of previously trained (fixed) agents~\cite{lanctot2017unified,vinyals2019grandmaster}. Instead, interaction graph describes a matchmaking schedule for co-training players, and thus imposes a continuous dynamical system of their evolution.

The contributions of this paper are threefold:\begin{enumerate}
\item We provide evidence of the importance of diversity at various levels of training in multi-agent, non-transitive games.
\item We introduce the interaction graph framework to describe the control of training interactions in populations.
\item We analyse the effect of different population graphs on the resulting populations for three different games. 
\end{enumerate} 

\section{Related work}

The research in this paper shares a lot of the ideas with early work on co-evolution, where agents are trained and evaluated against other agents that are being trained at the same time~\citep{paredis1995coevolutionary, morrison1999general, drezewski2003model}. Within co-evolution, a concept that is particularly related is that of Nash memory~\citep{ficici2003game, oliehoek2006parallel}. One could think of our approach as an amortised version of Nash memory with a single fixed-sized memory of strategies. In recent years co-evolution has gained a lot of interest in the context of population based training~\citep{jaderberg2019human, liu2018emergent, mckee2020diverse}. 

Also related to our line of work is research on diversity in reinforcement learning, which is often studied in the context of intrinsic motivation and skill discovery~\citep{eysenbach2018diversity, florensa2016stochastic, hausman2018learning, gregor2017variational}. The ability to discover diverse sets of options has also recently been considered from a generalisation and meta-learning point of view~\citep{li2018learning, frans2017meta, finnone}. The notion of strong diversity that we are interested in has also been studied in evolutionary computation under the name of quality diversity~\citep{pugh2016quality}.

Making use of graphs to capture interactions between agents in a population has also been proposed by Guestrin in the context of coordination graphs~\citep{guestrin2002coordinated}. However, this method assumes a graph of interactions between agents in a population is known and makes use of it to approximate a global value function. Our approach on the other hand only models separate value functions and does not require prior knowledge about the roles of agents in a populations or the structure of their interactions.

Finally, some of the graphs we suggest can be seen as co-training version of some previously introduced training scheduling algorithms like Self-Play~\citep{silver2017mastering}, PSRO~\citep{lanctot2017unified} and more recent work on PSRO rectified Nash~\citep{balduzzi2019open}. 
As such interaction graphs can be seen as a unifying language to describe agent matching algorithms in population-based learning.

\section{Definitions}

\label{sec:definitions}

\paragraph{\textbf{Populations of agents}} In this paper we focus on population-based methods. 
Our general setup consists of fixed-sized populations $\mathbf{P}$ of deterministic reinforcement learning (RL) agents that we train through pairwise interactions on two-player games.
From a game-theoretic point of view these deterministic agents represent strategies and mixtures of strategies can therefore be obtained as mixtures of agents. 

\paragraph{\textbf{Empirical evaluation matrix }} Following ~\citep{balduzzi2019open}, we define the antisymmetric function $\phi(\mathbf{v}, \mathbf{w}) = -\phi(\mathbf{w}, \mathbf{v})$ that evaluates the performance of any pair of opponents $\mathbf{v}$ and $\mathbf{w}$ on zero-sum two-player games. We say $\mathbf{v}$ beats $\mathbf{w}$ if $\phi(\mathbf{v}, \mathbf{w})>0$ and we capture all pairwise evaluations between agents of a population $\mathbf{P}$ and a second population $\mathbf{P'}$ in the empirical evaluation matrix $\mathbf{A}_{P, P'} = \{\phi(\mathbf{v}, \mathbf{w})\,,\; \mathbf{v} \in \mathbf{P}, \mathbf{w} \in \mathbf{P'}\}$. 

\paragraph{\textbf{Relative population performance }} From the evaluation matrix we can calculate the \emph{relative population performance} of $\mathbf{P}$ against $\mathbf{P}'$ as: $$u(\mathbf{P}, \mathbf{P}') = \mathbf{p} ^ \intercal \cdot \mathbf{A}_{\mathbf{P}, \mathbf{P}'} \cdot \mathbf{q}$$ where $\mathbf{p}$ and $\mathbf{q}$ are the Nash equilibria (see appendix for definition) for $\mathbf{P}$ and $\mathbf{P}'$ calculated from $\mathbf{A}_{\mathbf{P}, \mathbf{P}'}$.

\paragraph{\textbf{Effective diversity }} There are typically far more ways of acting incompetently than there are of acting effectively. Therefore, it is necessary to couple the notion of diversity with a measure of effectiveness.
Effective diversity~\citep{balduzzi2019open} measures the `useful' diversity of a population of agents $\mathbf{P}$ and is defined as: $$d(\mathbf{P}) = \mathbf{p} ^ \intercal \cdot {\mathbf{A}_{\mathbf{P, P}}}_+ \cdot \mathbf{p}$$ where $\mathbf{A}_{\mathbf{P, P}}$ is the payoff matrix between all of the agents in $\mathbf{P}$ against each other and $\mathbf{p}$ is the Nash equilibrium of $\mathbf{A}_{\mathbf{P, P}}$. Computing the rectifier ${\mathbf{A}_{\mathbf{P}}}_+$ implies that effective diversity is, roughly, a measure of the variety of ways agents can beat each other. Using the Nash distribution -- rather than, say, the uniform distribution -- ensures that effective diversity is relative to the best performing or least exploitable agents. In particular, if there is a single dominant agent then effective diversity is zero.

\paragraph{\textbf{Effective diversity in non zero-sum games}}

We can extend the effective diversity measure to non zero-sum games by subtracting the relative population performance from the payoff matrix first: 
$$d(\mathbf{P}) = \mathbf{p}^\intercal \cdot  {\mathbf{A}_{\mathbf{P, P}}-u(\mathbf{P}, \mathbf{P})}_+ \cdot \mathbf{q}.
$$ Intuitively, this measure `counts' the number exploits (payoffs above the relative performance of the population) that occur in the Nash.

\section{Strategic Diversity and Multi-Agent training}

Consider the task of training an agent that will compete in an online gaming league. In multi-agent environments we can distinguish between three functionally different roles that will be involved in training such an agent.

\begin{enumerate}
    \item \textbf{Learner}: this is the agent we are ultimately interested in and that will compete online. Our goal is to learn the right parameters for this agent from interactions with other agents.
    \item \textbf{Trainers}: a single agent or a set of agents that play against the learner to generate the experience that the learner uses for its updates. For example, a set of previously trained agents or human players.
    \item \textbf{Evaluators}: a single agent or a set of agents that are used to evaluate the performance of the agent. In the case of the online game this could be all other players worldwide, who the learner agent is ranked against.
\end{enumerate}

Given that in this paper we will be working with populations we will refer to these as \emph{learner population}, \emph{trainer population} and \emph{evaluator population} instead of as agents from now on, but the same principles hold.
While these three are functionally different, in practice a population can carry out several roles (often the trainers are used to evaluate the agent and in some cases the trainers are learning themselves). In the following we set out to prove that strategic diversity is necessary across \emph{all three types of populations} in order to obtain strong final performance of the learner population.

\subsection{Why evaluators should be strategically diverse}
\label{sec:evaluation}
In multi-player games the outcome is a function of several players. 
As such, we can only evaluate the performance $\phi(\mathbf{v}, \mathbf{w})$ of an agent or population of agents $\mathbf{v}$ relative to an opponent $\mathbf{w}$.
One way to define a measure of absolute performance $f(\mathbf{v})$ is to evaluate $\mathbf{v}$ against the entire agent space $A$ and choose the worst-case performance. 

\begin{equation}\label{eq:perf}
  \text{Performance:} \;  f_{A}(\mathbf{v}) := \min_{\mathbf{w} \in A} \phi(\mathbf{v}, \mathbf{w}) \,.
\end{equation}

This performance metric implies the learning goal is to find the least exploitable population, i.e. the population that maximises the objective~\eqref{eq:perf} (as is the case in the context of the mini-max theorem in game theory).

In practice we can only approximate $A$ by comparing the performance of population $v$ to a finite set of evaluator populations $\mathbf{P}^*$.
A diverse evaluator population $\mathbf{P}^*$ may contain approximate best response to $\mathbf{v}$, i.e., for any given $\mathbf{v}$ there exists an agent in $\mathbf{P}^*$ than can exploit $\mathbf{v}$ approximately as well as any agent of the entire space $A$.
The quality of the approximate evaluation of $\mathbf{v}$ against $\mathbf{P}^*$ can be viewed as a distance in functional space $d(f_{\mathbf{P}^*}, f_{A})$.
As such, in order to be able to estimate the performance of our populations correctly, our goal is to find a $\mathbf{P}^*$ that minimises it.
We therefore argue that strategic diversity of the evaluator population is instrumental for something as fundamental as \emph{defining} the performance of the learner population in the first place. 

\subsection{Why trainers should be strategically diverse}

If our goal is to train agents that perform well according to the notion of performance introduced in~\eqref{eq:perf} then we aim at solving a maxi-min problem of the form
\begin{equation}
    \max_{w \in A} \min_{v \in A} \phi(w, v) =: \max_{w \in A} f_{A}(w)  \,.
    \label{eq:best_response}
\end{equation}
This optimization problem is challenging as it has been shown that standard gradient methods may not converge even on very simple domains~\citep{balduzzi2018mechanics}. 
One way to contravene this issue is to consider the training of a main agent against a fixed population,
\begin{equation}\label{eq:min-max_popu}
    \max_{w \in A} \min_{v \in \mathbf{P}^*} \phi(w, v) =: \max_{w \in A} f_{\mathbf{P}^*}(w) \,.
\end{equation}
As argued in Section~\ref{sec:evaluation}, with a diverse enough population the latter objective is a reasonable approximation of the original optimization problem~\eqref{eq:best_response}. Moreover, almost everywhere, the gradients of $f_{\mathbf{P}^*}$ can be computed in a tractable way by using the identity
\begin{equation}
    \nabla  f_{\mathbf{P}^*}(w) =  \nabla_w \phi(w,\hat v(w)) \notag
\end{equation}
where $\hat v(w) \in   \mathbf{P}^*$ is a best response against $w$. 
Assuming that the payoff is differentiable, standard gradient based algorithms can be used to find an approximate solution of~\eqref{eq:best_response}. 
This type of population-based training is called prioritised fictitious self-play and was introduced by~\citep{vinyals2019grandmaster} to successfully train agents to play the game of StarCraft II. The particular instance of this method used by~\citep{vinyals2019grandmaster} solves a smooth version of~\eqref{eq:min-max_popu} where the agents sampled according to a soft-max of their performance against a `main agent'. \citep{vinyals2019grandmaster} put a major emphasis in obtaining a diverse population of `main agents' and `exploiters' to successfully train the `main agents' via prioritised fictitious self-play. 

\subsection{Why learners should be strategically diverse}
\label{sec:learner}
The need for learners to be strategically diverse is a result of the cyclic nature of non-transitive games.
A game is non-transitive if, whenever player $\mathbf{v}$ beats player $\mathbf{w}$ and $\mathbf{w}$ beats player $\mathbf{u}$, it does not follow that $\mathbf{v}$ beats $\mathbf{u}$ (rock, paper, scissors is perhaps the best-known example of a non-transitive game). 
These games involve strategic trade-offs: following a particular strategy will beat a subset of the remaining strategies but also lose to a different subset. 
As there is no "best" strategy, it is necessary for a population to cover a diverse set of strategies in order to perform well against the evaluator population at test time.

\begin{figure}[h]
\centering
\includegraphics[width=0.7\columnwidth]{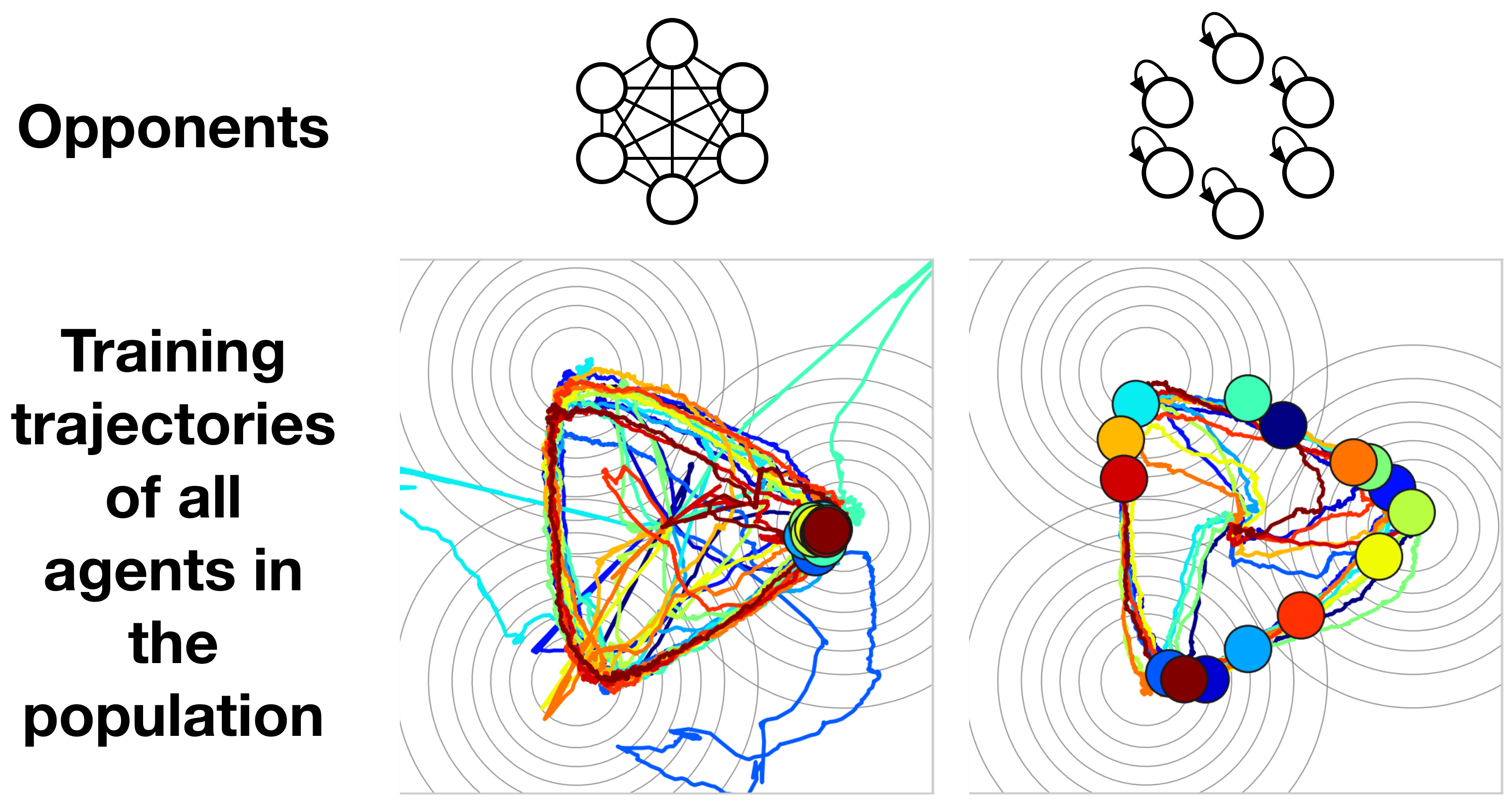}
\caption{\textbf{The effect of varying an agent's opponents on its training trajectory.} We train two populations with 15 agents each on a small cyclic game with three winning strategies (indicated by the three Gaussians in the plot). For one population we train all agents against all other agents in the population, while for the second population every agent is only trained against itself. Each line in the plot corresponds to the evolution of the policy of a single agent over training plotted in the 2D action space. The policy of each agent at the last time step is marked by a circle to visualise that some of the agents are moving through strategy space together while other's move independently.}
\label{fig:diversity_importance}
\end{figure}

\section{Interaction Graphs}

\subsection{Influencing agent training in multi-player games}
\label{sec:constructing}

Suppose we have a reinforcement learning agent $\mathbf{v}$ whose policy is parametrised by some non-linear function approximator $\pi_{\theta}^{\mathbf{v}}$ with parameters $\theta$.
The parameter updates of $\pi_{\theta}^{\mathbf{v}}$ are computed from the rewards $\mathbf{r}$ collected by $\mathbf{v}$.
In the case of multi-agent environments this reward will be some function $\mathbf{r} \sim g(\mathbf{\pi_{\theta}^{\mathbf{v}}}, \mathbf{\pi_{\theta}^{\mathbf{w}}})$ of the agent's policy $\mathbf{\pi_{\theta}^{\mathbf{v}}}$ as well as that of its opponent $\mathbf{\pi_{\theta}^{\mathbf{w}}}$. 
As a result, there are two ways of influencing the parameter updates of $\mathbf{v}$: (1) at an individual agent level via the parametrisation and update rule of the policy $\pi_\theta^\mathbf{v}$ or (2) at a population level by choosing different opponents $\pi_{\theta}^{\mathbf{w}}$.

In this paper we opt for the second population-level approach and define \emph{structured population-level objectives} via interaction graphs that specify the objectives of each agent in the population in terms of (mixtures of) other agents.
As a small motivating example of how strong the effect of training against different opponents can be on an agent we carry out the following experiment: we train two populations of RL agents on a non-transitive game.
Each of the agents is trained either against every other agent in the population (setting 1) or only against itself (setting 2).
We can visualise the different policies that the agents learn during training as trajectories in a 2D strategy space (Fig~\ref{fig:diversity_importance}).
As shown in the plots different sets of opponents lead to radically different exploration behaviours of the populations.
In the case of multi-agent training it is therefore imperative to consider the population-level interactions in addition to the agent-level objectives.

\subsection{Interaction graphs}
We have argued that high effective diversity is a desirable property of learner, trainer and evaluator populations.
Our goal is thus to develop systematic tools for training such diverse populations.
In practice, when training a new learner population one often does not have access to a trainer population.
In these cases the learners themselves are used as a trainers. 
Training every agent in a population against all other agents, however, would mean all agents have the same objective, which is not conducive to diverse behaviour.

Instead, we propose to train each agent against a specific subset of the agents in the same population that will act as that agent's particular trainer population.
We express these relations via interaction graphs.
The nodes of the interaction graph correspond to the agents of the population and the weights of the edges indicate to what extent the experience obtained against another agent $\mathbf{w}$ in the population is used to update the parameters of agent $\mathbf{v}$.
Crucially, the graph can be directed, meaning that $\mathbf{v}$ might care about beating $\mathbf{w}$, while the reverse might not be the case.
This allows for specialisation.
In addition the weights of the edges can change dynamically over training, such that the agents' objectives can be adapted to the strategies covered by the population.
Previous population-based training regimes can be expressed as a graph as well (all-to-all, self-play, PSRO etc). 
Other lines of research that are related to the concenpt of interaction graphs include~\citep{kearns2013graphical, lanctot2017unified, silver2018general}.

\subsection{Restricting the Information Flow in Populations}

By choosing only a subset of the population as the trainer for each agent we are restricting the information flow between the agents of the population.
In the following we show why this restriction can be beneficial for games that have non-transitive strategy spaces.

Suppose that $\phi(\mathbf{v}, \mathbf{w})$ is differentiable, and introduce the shorthand $\phi_{w}(\mathbf{v}) := \phi(\mathbf{v}, \mathbf{w})$ to emphasise that the weights of the opponent $\mathbf{w}$ are fixed and only $\mathbf{v}$ is to undergo training. Consider the gradients $\nabla_{\mathbf v} \phi_{\mathbf{w}_j}({\mathbf v})$ obtained against a set of opponents $\mathbf{w}_1,\ldots, \mathbf{w}_k$. Further, we say a game is \textbf{monotone} if $\phi$ factorises as $\phi(\mathbf{v}, \mathbf{w}) = \sigma(f(\mathbf{v}) - f(\mathbf{w}))$ where $f$ is a rating function that specifies the skill of each agent, and $\sigma$ is a monotone increasing function. If a game is monotone, then the performance of agents simply comes down to the difference in their ratings. The Elo rating system, widely used in Chess, assumes monotonicity.

\begin{prop}
    If the game is monotone then gradients against all opponents have nonnegative inner product with one another:
    \begin{equation}
        \Big\langle\nabla_{\mathbf v} \phi_{\mathbf{w}_j}({\mathbf v}),\nabla_{\mathbf v} \phi_{\mathbf{w}_k}({\mathbf v})\Big\rangle \geq 0.
    \end{equation}
    The inner product is zero if and only if there $\mathbf{v}$ is at a local maximum of the rating function $f$.
\end{prop}

\begin{proof}
    By the chain rule, $\nabla_{\mathbf v} \phi_{\mathbf{w}_j}({\mathbf v}) = \sigma'(f(\mathbf{v}) - f(\mathbf{w}_j)\cdot \nabla_{\mathbf v} f(\mathbf{v})$. Since $\sigma$ is monotone increasing, we have that $\sigma'$ is always positive, and the result follows.
\end{proof}

Proposition~1 shows that there are no strategic tradeoffs in monotone games. It makes little difference which opponent you train against, because the gradients all point in the same direction: infinitesimally improving against one opponent leads to infinitesimal improvement against all. The only difference arises from the magnitude of $\sigma'$. For example, if $\sigma$ is the sigmoid, then training against much better opponents does not help because $\sigma$ saturates and the gradient vanishes.

In contrast, in nontransitive games, training with -- and improving performance against -- one opponent can \emph{worsen} performance against other opponents. Roughly, training against rock makes you more like paper and so you perform worse against scissors, see \citep{balduzzi2019open} for a detailed example.
It follows that, in nontransitive games, training against mixtures can cause gradients to cancel out. It is therefore necessary to carefully control which gradients -- and so which opponents -- agents in a population are exposed to during training.

\section{Experiments}

\subsection{Graphs}
\label{sec:graphs}
We compare nine different graph structures to start characterising the effect of restricting the training interactions within populations. We distinguish between fixed and adaptive interaction graphs:

\textbf{Fixed interaction graphs} are defined at the beginning of training and remain fixed throughout:
\begin{enumerate}
    \item \textbf{All-to-all}: A fully connected graph, where every agent of the population trains against every other agent including itself. The objectives of all agents are the same and the information flow between agents is maximal. An example of agent populations trained with this type of interaction graphs is~\citep{jaderberg2019human}
    \item \textbf{Self-play}: every agent in the population trains independently against a past version of themselves. The flow of information is minimal. The AlphaGo agent~\citep{silver2017mastering}, for example, was trained using self-play.
    \item \textbf{Cycle}: One directed cycle of the same length as the number of agents in the population. The motivation behind this type of graph is to reflect the cyclic nature of the game in the training regime.
    \item \textbf{Hierarchical cycle}: All agents except for one are part of a directed cycle and the final agent has directed connections to all. This includes the cyclic structure from 3, and an additional agent that learns best response to everyone in the cycle.
    \item \textbf{PSRO}: All agents in the population are numbered and every agent plays all the agents with a smaller index than itself. The idea behind this hierarchical structure is to have increasing levels of competence~\citep{lanctot2017unified}.
\end{enumerate}

\textbf{Adaptive interaction graphs} start out fully connected and the edges are then continuously updated during training according to some metric such as relative performance against the other agents in the population:
\begin{enumerate}
\setcounter{enumi}{5}
    \item \textbf{Play better:} every agent $\mathbf{v}$ only trains against those agents in the population that it is is losing against as reflected in the payoff matrix.
    \item \textbf{Play worse:} every agent $\mathbf{v}$ only trains against those agents in the population that it is beating according to the payoff matrix.
    \item \textbf{Play worse and self:} Same as play worse, but agents also train against themselves.
    \item \textbf{PSRO Rectified Nash:} introduced in~\citep{balduzzi2019open} and essentially play worse, but only for those agents that have support in the Nash. The rest does self-play.
\end{enumerate}

\begin{figure}[h]
\centering
\includegraphics[width=0.7\columnwidth]{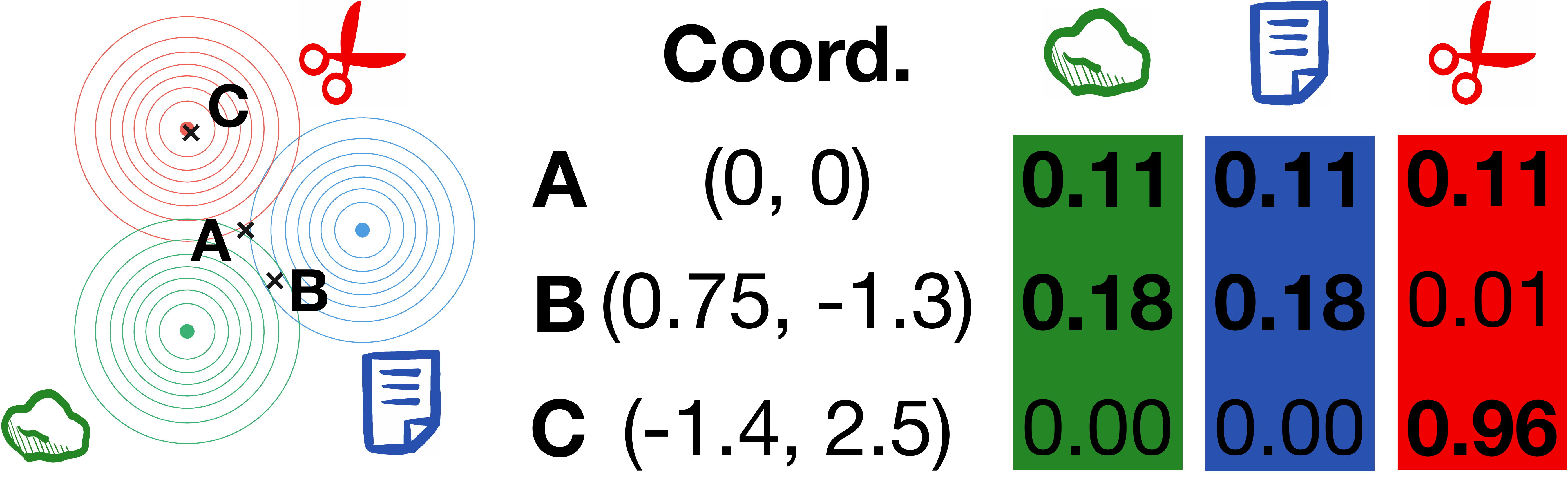}
\caption{\textbf{Diagram of the GMM-RPS game with three sample strategies A, B and C.}
}
\label{fig:gmm_rps}
\end{figure}

\subsection{Environments}

\paragraph{\textbf{GMM-RPS}}
The motivation behind this continuous variant of the classic Rock-Paper-Scissors (RPS) game is to create a simple game that combines cyclic and transitive components.
At a higher level, the game has cyclical dynamics just like RPS.
In addition there is a transitive element of strength: for example, a stronger rock can beat a weaker rock.
The game is illustrated in Fig~\ref{fig:gmm_rps} and for a more detailed description see Section~\ref{sec:rps-gmm} in the appendix. 
We distinguish between GMM-RPS(3), which has 3 modes (like rock, paper and scissors in RPS) and GMM-RPS(7) which has 7 modes and as a result a stronger cyclic component than GMM-RPS(3) as the modes lie closer to each other.

\begin{figure*}[h]
\centering
\includegraphics[width=0.8\textwidth]{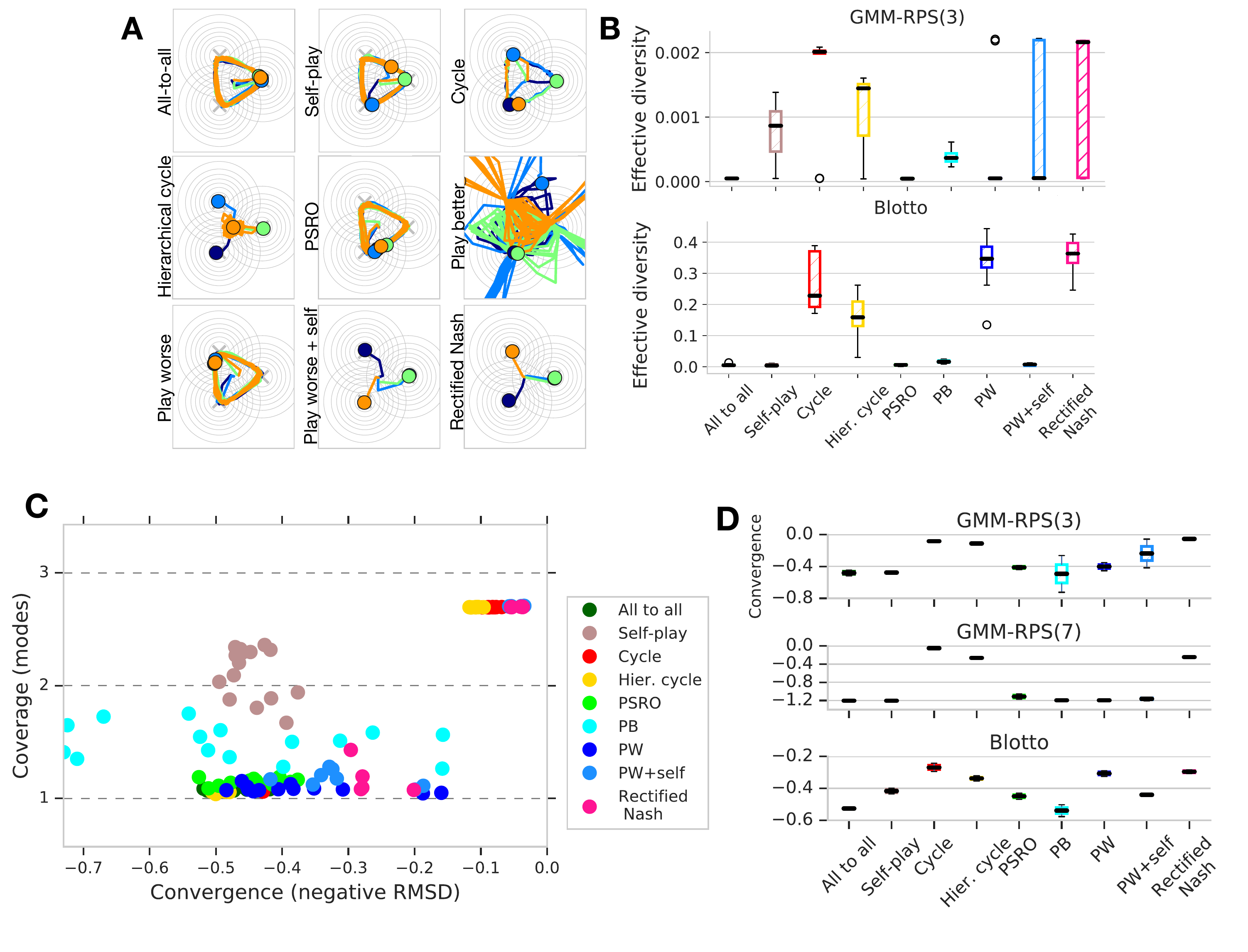}
\caption{\textbf{Training trajectories and effective diversity of populations trained with different interaction graphs}. \textbf{A}: Training trajectories of four-agent populations trained on GMM-RPS using different interaction graphs. Each of the four coloured lines shows the mean policy over training time for each agent. The policy at the last step of training is indicated by a circle for each of the agents. The three grey Gaussians correspond to the rock, paper and scissors modes. \textbf{B}: the effective diversity obtained with the 9 different interaction graphs averaged over 15 populations. C: Convergence of individual agents as measured by the negative RMSD versus coverage of the GMM-RPS(3) action space. In the context of this game coverage corresponds to the number of modes that the agents of a population cover combined. Each dot corresponds to a population and the colour indicates what interaction graph it was trained with. The convergence value is averaged over all four agents of the population. D: The convergence on all 3 games for populations trained using different interaction graphs.}
\label{fig:second}
\end{figure*}

\paragraph{\textbf{Colonel Blotto.}} 
Colonel Blotto is a two-player, zero-sum resource distribution game. Two players are given $k$ tokens that they distribute simultaneously over $a$ areas. Whichever player has more tokens on an area wins that area and the player that wins the most areas wins the overall game. Colonel Blotto is well-studied in game-theory where it's usually of particular interest because of its highly non-transitive strategy space. In this paper we consider a variant of Blotto with $k=100$ and $a=7$.

\paragraph{\textbf{Starcraft.}}
Finally, we include some experiments in the StarCraft II environment~\citep{vinyals2017starcraft}.
StarCraft II is a real-time strategy 2-player game with highly non-transitive game dynamics. 
In addition to being significantly more complex than the previously described games it also has a temporal element that the other two lack.
The motivation behind including this benchmark is to rest whether any of the findings obtained on the simple one-step games translate to more complex temporal games.

\subsection{Evaluation metrics}

In order to characterise the behaviours of populations obtained with different interaction graphs we consider several metrics. On a qualitative level we analyse the training trajectories of the agents by plotting the policy of the agent as it evolves over training iterations. In the case of GMM-RPS the strategies live in 2D space so the plots are easily interpretable. On a quantitative level we measure the effective diversity and relative population performance (RPP) of the different populations. In order to measure the RPP we need an evaluator baseline. We use learned populations with high and low measured effective diversity as well as a `ground truth' population containing all the strategies in the Nash if they are known. Finally, we also measure the convergence of the agents' policies to evaluate whether they get stuck in cycles.

\subsubsection{StarCraft}

Given the complexity of Starcraft we evaluate the populations on different but related metrics. 

\paragraph{\textbf{Performance}}
To measure overall performance, we report maximum (over agents in the population) average win-rate against the test set $\mathcal{T}$. A performance of the population is the maximum of the performances among the population. It can be seen as an RPP metric where the test population is a single agent, that players a uniform mixture of the test set.
\begin{equation}
    \begin{aligned}
\mathrm{perf}_\mathrm{sc2}(\mathcal{P}) &= \max_{a \in \mathcal{P}} \mathbb{E}_{b\sim \mathrm{U}(\mathcal{T})} \mathrm{Prob}[a\text{ beats } b] \\&= \mathrm{RPP}(\mathcal{P}, \{\mathrm{U}(\mathcal{T})\}).
    \end{aligned}
\end{equation}
We use this version of the performance, as none of our population was able to get non-zero win rate against entire test set, and so $\mathrm{RPP}(\mathcal{P}, \mathcal{T}) = 0$ for each $\mathcal{P}$. 

\paragraph{\textbf{Diversity}}
We first define marginal distributions of wheter given units were produced through last 7e8 steps $\phi_\mathrm{units}(a) \in [0,1]^{19}$ (there are $19$ Protoss units in the game of SC2, see Fig.~\ref{fig:sc2_units} in the Appendix) by each agent $a$ in the population $\mathcal{P}$.  Then we use a proxy for diversity, where the distance between two agents $a$ and $b$ is simply the maximum squared distance over units production probability. And to define a population level measure we take minimum of such a distance over all possible pairs of agents.
\begin{equation}
    \begin{aligned}
\mathrm{div}_\mathrm{sc2}(\mathcal{P}) &=  \min_{a \neq b \in \mathcal{P}} \max_u
 [\phi_\mathrm{units}(a)[u] -  \phi_\mathrm{units}(b)[u]]^2
 \end{aligned}
\end{equation}

\paragraph{\textbf{Coverage}}
Given the complexity of the game of SC2 it is really hard to derive a meaningful measure of strategic coverage. However, we argue that in order to be prepared for every possible scenario in the game, at least agents need to perceive every unit in the game. We thus look at a simple proxy, fraction of units that a population as a whole, is creating over series of episodes, to get a rough notion of strategic coverage. We count unit as present, if through last 7e8 training steps it was created at least 5\% of the time.
$$
\mathrm{cov}_\mathrm{sc2}(\mathcal{P}) = 
\tfrac{1}{19} \left | \left \{u : \left [ \tfrac{1}{|\mathcal{P}|} \sum_{a \in \mathcal{P}} \phi_\mathrm{units}(a)[u] \right ] \geq 5\% \right \} \right|.
$$

\subsection{General population training}

We train populations of four A2C agents on the GMM-RPS and Blotto environments.
Every 10000 frames sampled from environment interactions we update the interaction graph which determines the opponents for the actor.
To sample opponents from the graph we first select agent 1 from the population with uniform probability and then sample agent 2 with probability proportional to the weights of the edges that go into the node of agent 1. 
This experience is then only used to update the weights of agent 1.
\subsubsection{StarCraft}
For the StarCraft experiments we use the PySC2 environment~\cite{vinyals2017starcraft}. Each experiment uses a population of 4 agents, trained according to the population graph (as opposed to the setup of \citep{vinyals2019grandmaster} we do not have checkpoints or exploiters). We use exactly the same architecture, action and observation spaces as well as pretraining procedure as \citep{vinyals2019grandmaster}. However, we only trained agents to play one race -- Protoss, and just one map -- Kairos Junction, to limit the game complexity. Each job was ran for 3 billion learner steps.

We use a set of pre-trained simple agents that we refer to as test agents, that were hard coded to follow various strategies, following the description of \citep{vinyals2019grandmaster}.

\section{Results}

\subsection{Training trajectories}
We record the actions produced by the policy at the time of the interaction graph update throughout training. 
These actions can then be plotted on the 2D plane for the GMM-RPS games to visualise how the agents' policies are being updated over time.
We plot the actions of each individual agent as a line and mark the last training step with a circle as shown in Fig.~\ref{fig:second} for each interaction graph on the game of GMM-RPS(3). In order to show the consistency of the observed exploration behaviour across runs we plot four repetitions of a population trained using the same interaction graph in Figure~\ref{fig:big_trajectories} in the Appendix. Overall the behaviour observed for each of the graphs is consistent. We can also see from these plots that the graphs that encourage specialisation (such as rectified Nash or Hierarchical Cycle) can also end in cycling behaviour. In this figure we also show the trajectories for GMM-RPS(7). The cyclic component is stronger in this environment as the different modes are closer together. This is also reflected in the training trajectories as there is a higher number of populations that end up cycling. The graphs that encourage specialisation also display cycling behaviour but rather than cycling together what we observe is that the agents switch modes one at a time, while still covering a wider spread of the modes.

\subsection{Diversity and performance}

We train 15 populations on each of the nine interaction graphs described in section~\ref{sec:graphs}. 
Out of these 15 populations we choose the 10 with the highest effective diversity for each interaction graph.
We calculate the effective diversity for each of the 90 populations using the equation from Section~\ref{sec:definitions} and we take the average over the last 50 logged training iterations.
The effective diversity results for the different interaction graphs on GMM-RPS(3) and Blotto are shown in Figure~\ref{fig:second} and on GMM-RPS(7) in Figure~\ref{fig:gmm-rps7} in the Appendix.

Similarly, we measure the relative population performance as described in section~\ref{sec:definitions}.
When measuring the relative performance, however, we also require an evaluation population. 
As in the previous section we consider three evaluation populations: ground truth, high diversity and low diversity populations.
We only include the ground truth population for the GMM-RPS games as the strategy space of Blotto is too large to enumerate.
We plot the relative population performance against the effective diversity in Figures~\ref{fig:first} for on GMM-RPS(3) and~\ref{fig:gmm-rps7} in the Appendix for on GMM-RPS(7).

\subsection{Convergence vs Coverage}
As discussed in Section~\ref{sec:results}, convergence of the individual agents is not as important for performance as population-level convergence.
This means individual agents may still cycle as long as the population as a whole covers all strong strategies. 
To test this we quantify the coverage of different populations.
For GMM-RPS(3) in particular we measure how many of the Gaussian modes a population covers as a whole.
As shown in Fig~\ref{fig:second}C for this particular game populations that have good coverage tend to consist of agents which have individually converged.
This however is not always the case. 
Populations trained with the `self-play' interaction graph, for example, achieve comparable coverage, despite of all agents cycling throughout training. 
On the other hand there are also populations with high convergence that cover only a single mode.
Furthermore, among the populations with the highest coverage there are varying degrees of convergence.
In these cases most agents have converged, but a few are still cycling (in the case of the `cycle' graph, for example, all agents are slowly cycling but only one at at time, which results in a high effective convergence, despite the fact that agents are still cycling).
We have also plotted the convergence for the three different games (Fig~\ref{fig:second}D). 
As expected, populations that cycle in GMM-RPS(3) cycle even more in GMM-RPS(7) as the cyclic component is stronger.
In addition, the interaction graphs that encourage convergence of individual agents in GMM-RPS also do so in Blotto. 

\begin{figure}[h]
\centering
\includegraphics[width=0.6\columnwidth]{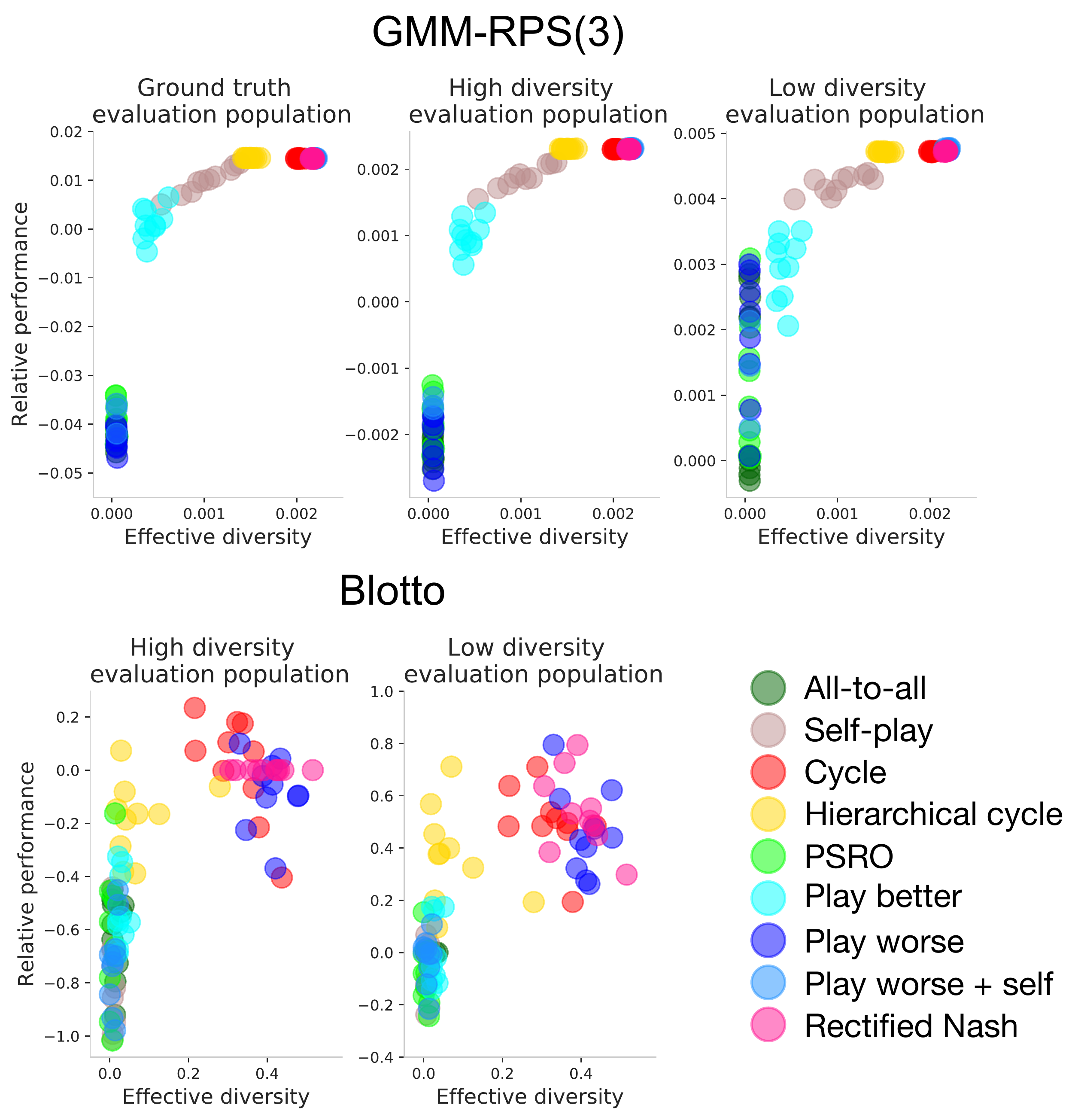}
\caption{\textbf{Relative performance vs effective diversity for populations trained using different interaction graphs}. For each interaction graph we train 10 populations and plot the effective diversity against the RPP averaged over the last training iterations. Each marker corresponds to one population and is colour coded according to the interaction graph.}
\label{fig:first}
\end{figure}

\subsection{Starcraft}
The learning curves of the populations trained on the different interaction graphs are shown in Fig.~\ref{fig:sc2}. 
The graphs that resulted in high performance on the simpler games do not necessarily perform well on the significantly more complex game of StarCraft. 
All-to-all, for example, performs very well, where as a graph as simple as the Cycle does not perform well on a game with high strategic complexity.
The scores for performance as well as diversity and coverage are summarised in Table~\ref{tab:sc2_details}.
From these results we can also see that in the case of StarCraft plain diversity does not directly translate to performance.
This is reflected both in the coverage as well as the diversity numbers that which are not necessarily correlated with the performance scores obtained.

\section{Discussion}
\label{sec:results}

We characterise the differences across all nine interaction graphs by looking at the evaluation metrics mentioned above. In the following we summarise the main findings.

\paragraph{\textbf{Diverse evaluator populations estimate performance more accurately.}} 
As shown in Figure ~\ref{fig:first} the perfomance of the 90 populations evaluated by an evaluator population with high effective diversity matches the ground truth (absolute) performance better than that of a low diversity evaluator population. This effect is even stronger as the game increases in complexity (see results on GMM-RPS(7) in Figure~\ref{fig:gmm-rps7} in the appendix). These results aligns with our claims in Section~\ref{sec:evaluation}.

\begin{table}[h]
    \centering
    \begin{tabular}{lccc}
    \toprule
    Method & Performance & Coverage & Diversity \\  
    \midrule
          All-to-all & 47\% & 89\% &  39\% \\
         Hier. Cycle & 46\% & 74\% &  29\% \\
               Cycle & 30\% & 74\% &  12\% \\
          Rect. Nash & 43\% & 63\% &  3\% \\
           Self-play & 44\% & 53\% &  0\% \\
         \bottomrule
    \end{tabular}
    \caption{Results of the SC2 experiments.}
    \label{tab:sc2_details}
\end{table}

\paragraph{\textbf{Diverse learner populations perform better.}} 
As postulated in Section~\ref{sec:learner} high effective diversity seems to be correlated with strong performance. This is reflected for both GMM-RPS(3) and Blotto in Figure~\ref{fig:first} and for GMM-RPS(7) in Figure~\ref{fig:gmm-rps7} in the appendix.

\paragraph{\textbf{Graphs influence the training behaviour of populations.}} 

The training trajectories that result from training on different interaction graphs vary drastically as shown in Figure~\ref{fig:second} for GMM-RPS(3) (more trajectories as well as the trajectories on GM-RPS(7) can be found in Figure~\ref{fig:big_trajectories} in the appendix). Given the simplicity of the game most populations display one of two possible behaviours: the agents either all synchronise within a population or they cover different areas of the action space independently. 
In general, interaction graphs that allow for cyclic training interactions cover all modes, while those that don't contain cycles end up cycling.

\paragraph{\textbf{Graphs influence the effective diversity of populations.}} 
We can further quantify this behaviour by measuring the average effective diversity obtained by populations trained on the different interaction graphs. The right hand plots in Figure~\ref{fig:second} confirm the trend observed in the trajectory plots to the left, whereby interaction graphs that can contain cyclic relations between agents have a larger spread across action space. For GMM-RPS and Blotto this spread also translates to higher effective diversity.

\textbf{Graphs with cycles encourage specialisation and increase effective diversity in simple non-transitive games.} As reflected in the training trajectories (Fig.~\ref{fig:second}) the directed nature of cycles allows individual agents to focus on a subset of the population (that does not necessarily focus on them) and thus to specialise. Populations trained with undirected graphs, on the other hand, tend to collapse to the same strategy as the symmetry in the connections means agents have the same objective. As a result populations trained with cyclic interaction graphs have higher effective diversity (Fig.~\ref{fig:second}B).
    
\textbf{A fixed graph structure is powerful when it matches the underlying game structure, otherwise adapting graphs might be a better choice.} The hierarchical cycle, for example, works well on the RPS-like games as it matches the underlying structure. It does not, however, perform as well on Blotto which has a richer strategy structure (compare effective diversity (Fig.~\ref{fig:second}) and RPP (Fig.~\ref{fig:first}) across games). The adaptive graphs, on the other hand, find a good approximation in either case.
    
\textbf{Focusing on those that are better than you makes you less exploitable and focusing on those that are worse than you makes you a better exploiter.} Focusing on opponents that beat you involves learning best response against a more diverse set of strategies which encourages agents to be more robust. Playing against those you are beating already, on the other hand, allows agents to specialise. As a result populations might become more exploitable as they might get stuck on weak enemies. The rectified Nash interaction graph could be one way of remedying this, as agents only specialise on other agents that have support in the Nash.
    
\textbf{Individual convergence is not as important as population-level convergence for diversity and coverage.} Individual agents may cycle as long as all important strategies are covered. 
We can see this in Figure~\ref{fig:second}C  where the `One cycle' interaction graph has low individual convergence but good coverage of the three GMM-RPS modes at all times.

\textbf{When moving to significantly more complex environments some fundamental insights hold, but some do not.} 
We have chosen simple games as a starting point for our analysis. While most insights hold across these games, they may not translate to significantly more complex games such as StarCraft. 
In fact, some intuitions, e.g. the usefulness of directed graphs, the fact that the wrong fixed graph can hinder learning or that focusing on agents you beat allows you to specialise seem to agree with the results obtained on StarCraft. However, it is also clear that one should be careful to translate graphs or particular methods directly from very simple environments to more complex ones. Stark difference in game dynamics might lead to unexpected failure modes (e.g. the collapse of rectified Nash onto a single Nash agent that it can't recover from) or unforeseen successes~(e.g. the ability of 'All-to-all' to explore the strategy space). 

\begin{figure}[htb]
\centering
\includegraphics[width=0.8\textwidth]{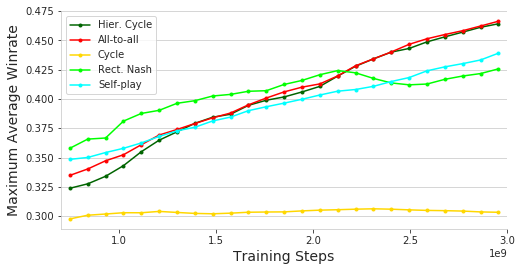}
\caption{Learning curves of 4-agent populations over training on StarCraft II.}
\label{fig:sc2}
\end{figure}

\subsection{Outlook}

A large part of current machine learning algorithms rely on optimisation. As a result, the community has developed a powerful repertoire of tools -- such as gradient descent, reinforcement learning, and evolutionary algorithms -- for minimizing losses and maximizing discounted future rewards. Multi-agent systems break this paradigm, because there is \emph{no longer a fixed objective}. Performance depends on the behavior of the other agents or opponents in the system and there are often many ways of behaving well. 

It is therefore necessary to design algorithms that search the space of possible behaviors and find diverse, effective strategies. 
In this paper, we have explored population-based learning algorithms where the objectives of agents in the population are specified by \emph{interaction graphs}. 
We find that training with certain graphs consistently yields populations of agents that are both stronger and more diverse than naive baselines.  
Interactions graphs provide a basic framework for reasoning about effective population-level objectives that encourage diversity and improve overall performance.

\bibliographystyle{plain}

\clearpage

\section{Appendix}
\subsection{Definitions}

\paragraph{\textbf{Nash on empirical games}} 
Given an empirical evaluation matrix $\mathbf{A}$ as above, define the Nash as the minmax solution to matrix game:
\begin{equation}
    \min_{\mathbf p}\max_{\mathbf q} {\mathbf p}^\intercal {\mathbf A}{\mathbf q},
\end{equation}
where we write $({\mathbf p}^*, {\mathbf q}^*)$ for the corresponding distributions.
In general there could be more than one pair of distributions that form a Nash equilibrium. There does exist a unique \emph{maxent} Nash equilibrium with nice properties, see~\citep{balduzzi2018re} for details. In practice, we find it convenient to use the Nash returned by an LP solver. If the matrix is \emph{antisymmetric} (which occurs for example when ${\mathbf A}$ is the empirical evaluation matrix of a population playing against itself) then it is shown in~\citep{balduzzi2018re} that there are symmetric Nash equilibria of the form $({\mathbf p}^*,{\mathbf p}^*)$ -- i.e. where the Nash involves both meta-players choosing the same distribution.

\subsection{GMM-RPS Environment}
\label{sec:rps-gmm}

The motivation behind this version of continuous RPS is to create an environment that has both a cyclic as well as a transitive component to it. 
The game consists of a 2D plane with three equidistant bivariate Gaussians that represent rock, paper and scissors respectively. A strategy corresponds to a point on that plane, defined by a pair of 2D coordinates. In Figure~\ref{fig:gmm_rps} we visualise three trategies A-C on the 2D plane along with their coordinates. These 2D coordinates are translated into rock-paper-scissors (RPS) proportions by measuring the pdf under each of the three Gaussians for every point (the  RPS weights for A-C are shown on the right). In this particular setup A is predominantly a rock-strategy, B an equal mixture between rock, paper and scissors and C an equal mixture between rock and paper.
GMM-RPS(3) deviates from continuous RPS in two ways:

\begin{enumerate}
\item We define a non-linear mapping from $\mathbb{R}^2$ to $\mathbb{R}^3$ between the agent's actions and the continuous RPS strategies. The idea behind this is to obtain some local optima around the pure strategies of rock, paper and scissors. More specifically, we define the action space to be a 2 dimensional plane and place three equidistant bivariate Gaussians on this plane that each represent one of the pure strategies (rock, paper and scissors). Every action is therefore a point on the plane that is mapped into the $\mathbb{R}^3$ space of continuous RPS by measuring the probability density function under each of the three Gaussians (see Fig.~\ref{fig:gmm_rps}.
\item We also add a transitive component to the game by defining the game matrix to be:

\begin{equation}
    \begin{bmatrix}
    0.5 & 1 & -1 \\
    -1 & 0.5 & 1 \\
    1  & -1 & 0.5
    \end{bmatrix}
\end{equation}

the 0.5 on the diagonal means that a strategy with a stronger rock, for example, will beat a weaker rock.
\end{enumerate}

GMM-RPS(7) is equivalent but with seven Gaussians instead of three so the mapping is from $\mathbb{R}^2$ to $\mathbb{R}^7$ and every strategy loses against half of the remaining strategies and beats the rest.


\subsection{Additional Results}

\subsubsection{GMM-RPS(7)}
We run the same experiments that were carried out for GMM-RPS(3) on GMM-RPS(7). As in the main text we report the average effective diversity obtained with each interaction graph as well as a plot of the RPP against effective diversity with different evaluator populations (see Figure~\ref{fig:gmm-rps7}). 
The results obtained on this environment match those observed with GMM-RPS(3) whereby graphs that allow for cycles have a higher diversity and higher diversity is correlated with higher RPP. 
As before, evaluator populations that are more diverse themselves are better at estimating the performance of leaner populations than non-diverse ones.

\begin{figure}[h]
\centering
\includegraphics[width=0.5\columnwidth]{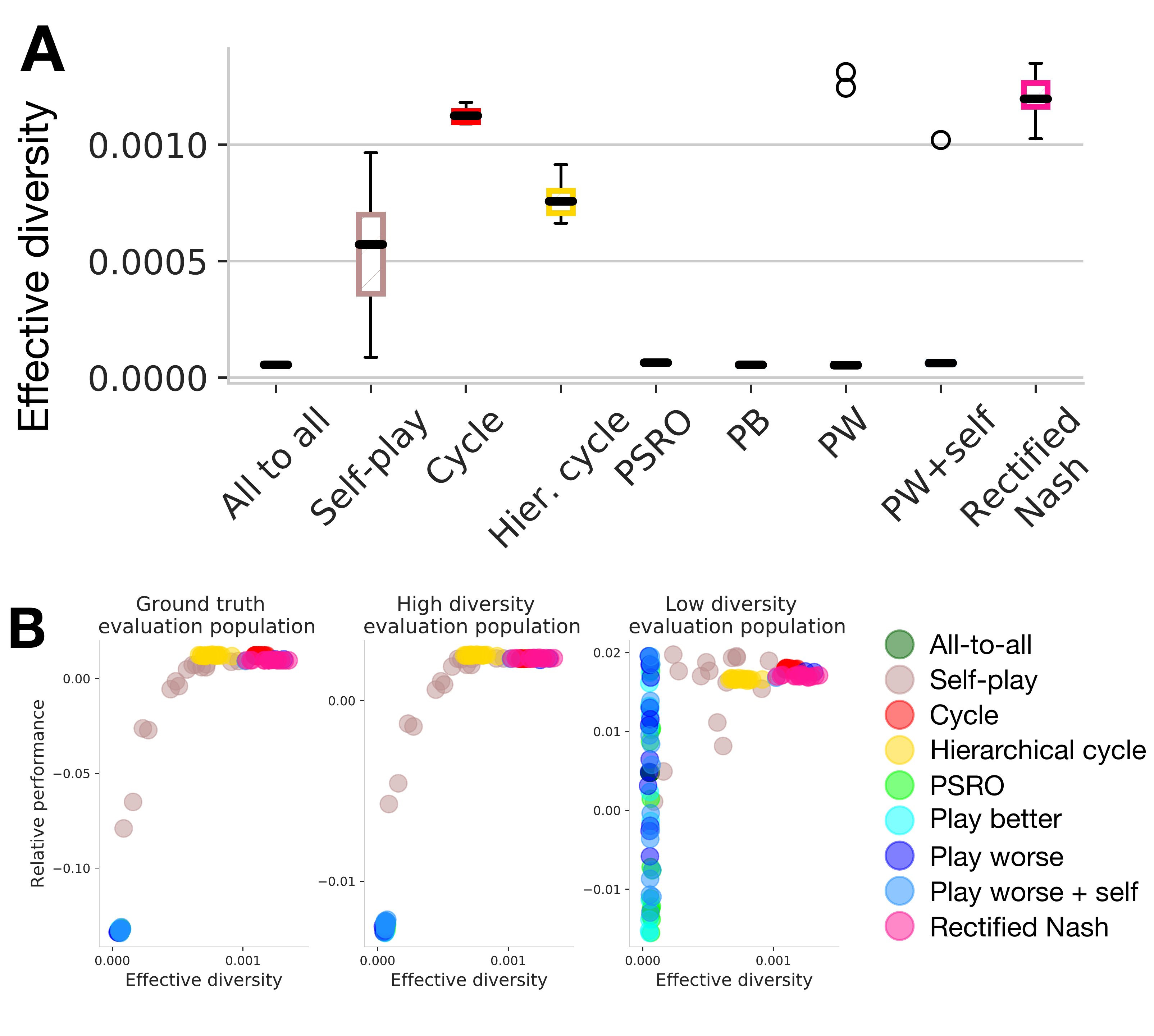}
\caption{\textbf{A: Effective diversity and B: RPP vs effective diversity for nine different inteaction graphs on GMM-RPS(7)}. }
\label{fig:gmm-rps7}
\end{figure}

\subsubsection{Diversity in StarCraft}
We visualise the diversity obtained with different interaction graphs by plotting the marginal distribution of units produced by each agent in Fig.~\ref{fig:sc2_units}.

\begin{figure*}[h]
\centering
\includegraphics[width=7cm]{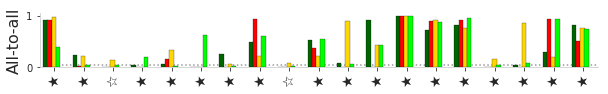}
\includegraphics[width=7cm]{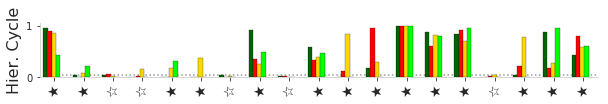}\\
\includegraphics[width=7cm]{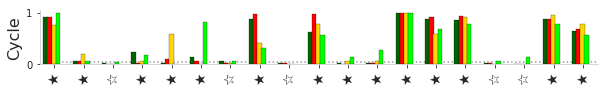}
\includegraphics[width=7cm]{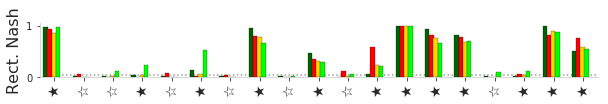}\\
\includegraphics[width=7cm]{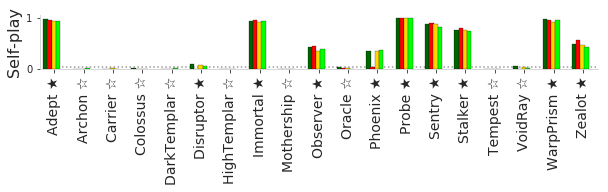}
\caption{Visualisations of $\phi_\mathrm{units}$, marginal distribution of units produced by each agent in the population. Colors correspond to each separate agent, and each bar represents probability of constructing a specific Protoss unit in last 7e8 training steps. Dotted line corresponds to 5\% cutoff from definition of coverage. Black star next to unit name means that it is considered present in terms of coverage definition, and white star -- that it is not.}
\label{fig:sc2_units}
\end{figure*}

\subsubsection{Additional trajectories for GMM-RPS}
We show the training trajectories for several repetitions of training a population of agents with the different interaction graphs in Fig.~\ref{fig:big_trajectories}. 
In these plots every row corresponds to a different experiment and every column to a different interaction graph.

\begin{figure*}[h]
\centering
\includegraphics[width=0.8\textwidth]{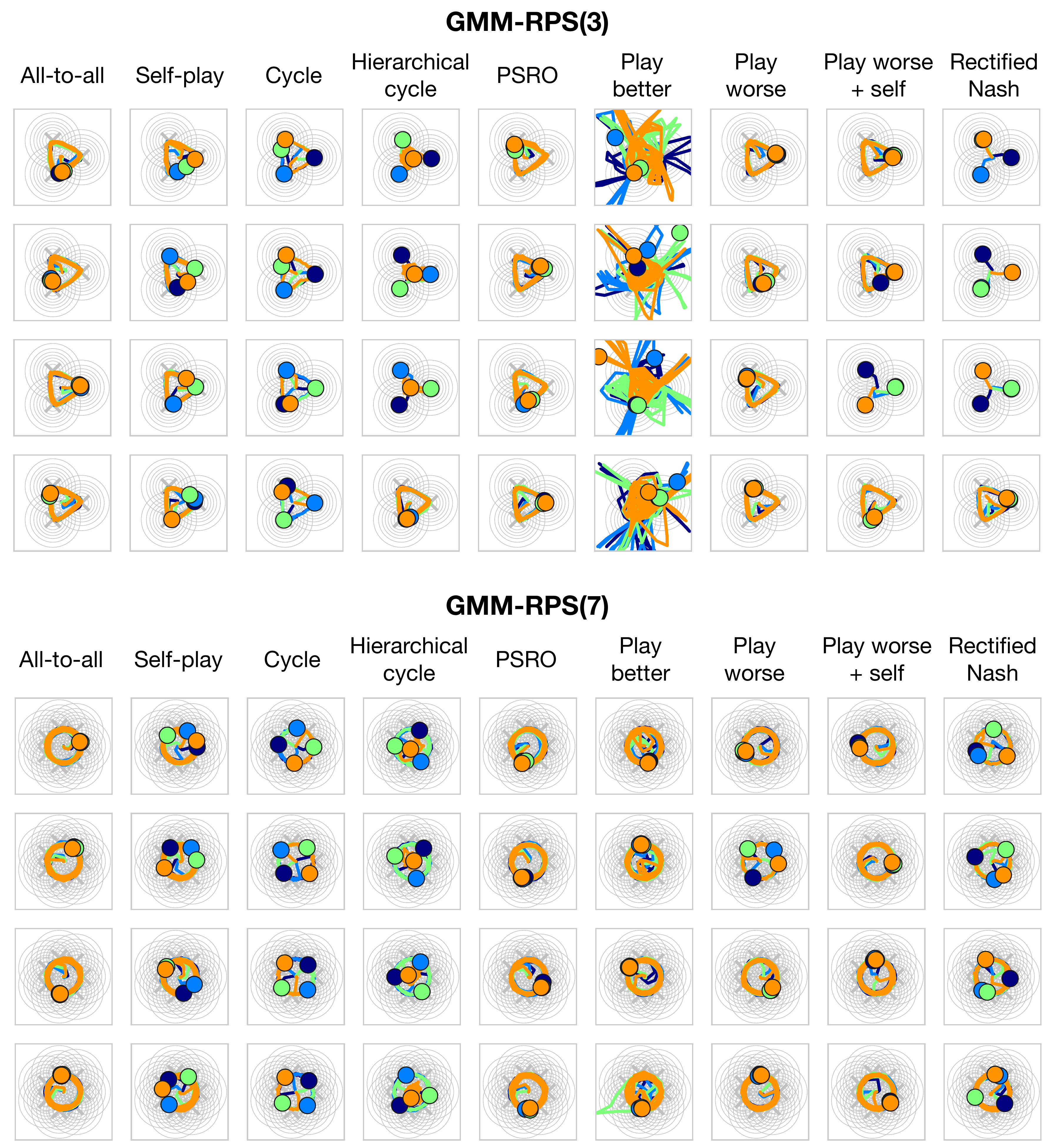}
\caption{\textbf{The trajectories of the 4-agent populations over training on GMMRPS-(3) and GMMRPS-(7) for different interaction graphs}. Each agent's trajectory is visualised using a different colour and the final training step of each is marked by a circle. Each figure corresponds to one populations. Populations in the same column are trained using the same interaction graph. The rows are different repetitions of the population training.}
\label{fig:big_trajectories}
\end{figure*}

\end{document}